%% file: ms.tex
\pdfoutput=1
\documentclass{article}

\usepackage{microtype}
\usepackage{graphicx}
\usepackage{subfigure}
\usepackage{booktabs} 

\usepackage{hyperref}

\usepackage{amsmath} 
\usepackage{amssymb} 
\usepackage{amsthm}
\usepackage[noabbrev,capitalize]{cleveref} 
\usepackage{enumitem}
\usepackage[T1]{fontenc}
\usepackage{svg}
\usepackage{tikz}
\usepackage{placeins}
\graphicspath{ {figures/} }
\usepackage[hang]{footmisc}
\usepackage{mathtools, nccmath}
\usepackage{verbatim}
\DeclareGraphicsExtensions{.eps,.pdf,.png}

\newtheorem{theorem}{Theorem}
\crefname{theorem}{Theorem}{Theorems}
\Crefname{theorem}{Theorem}{Theorems}

\newtheorem{theoremcorollary}{Corollary}[theorem]
\crefname{theoremcorollary}{Corollary}{Corollaries}
\Crefname{theoremcorollary}{Corollary}{Corollaries}

\newtheorem{definition}{Definition}
\crefname{definition}{Definition}{Definitions}
\Crefname{definition}{Definition}{Definitions}

\newtheorem{lemma}{Lemma}
\crefname{lemmacorollary}{Lemma}{Lemmas}
\Crefname{lemmacorollary}{Lemma}{Lemmas}

\newtheorem{lemmacorollary}{Lemma}[lemma]
\crefname{lemma}{Lemma}{Lemmas}
\Crefname{lemma}{Lemma}{Lemmas}

\theoremstyle{definition}

\crefname{example}{Example}{Examples}
\Crefname{example}{Example}{Examples}

\usepackage[accepted]{custom}

\icmltitlerunning{Optimal Model Averaging: Towards Personalized Collaborative Learning}

\begin{document}

\twocolumn[
\icmltitle{Optimal Model~Averaging:\\ Towards Personalized Collaborative~Learning}
\icmlsetsymbol{equal}{*}

\begin{icmlauthorlist}
\icmlauthor{Felix Grimberg}{EPFL}
\icmlauthor{Mary-Anne Hartley}{EPFL}
\icmlauthor{Sai P.~Karimireddy}{EPFL}
\icmlauthor{Martin Jaggi}{EPFL}
\end{icmlauthorlist}

\icmlaffiliation{EPFL}{School of Computer and Communication Sciences, EPFL, Lausanne, Switzerland}

\icmlcorrespondingauthor{Felix Grimberg}{felix.grimberg@hotmail.com}

\icmlkeywords{Federated Learning, Model Personalization, Personalized Federated Learning, Model Interpolation, Model Averaging, Bias-Variance-Trade-Off, Mean Estimation, Weighted Model Averaging, Model Mixing, Data Interpolation}
\vskip 0.3in
]

\printAffiliationsAndNotice{}

\begin{abstract}
In federated learning, differences in the data or objectives between the participating nodes motivate approaches to train a personalized machine learning model for each node.
One such approach is weighted averaging between a locally trained model and the global model.
In this theoretical work, we study weighted model averaging for arbitrary scalar mean estimation problems under minimal assumptions on the distributions.
In a variant of the bias-variance trade-off, we find that there is always some positive amount of model averaging that reduces the expected squared error compared to the local model, provided only that the local model has a non-zero variance.
Further, we quantify the (possibly negative) benefit of weighted model averaging as a function of the weight used and the optimal weight.
Taken together, this work formalizes an approach to quantify the value of personalization in collaborative learning and provides a framework for future research to test the findings in multivariate parameter estimation and under a range of assumptions.
\end{abstract}

\input{0_Body}

\bibliography{ms}
\bibliographystyle{custom}

\input{0_appendix}
\end{document}

%% file: 0_Body.tex
\newcommand{\guillemots}[1]{\guillemotleft{} #1 \guillemotright{}}
\newcommand{\brackets}[1]{\left( #1 \right)}
\newcommand{\squareBrackets}[1]{\left[ #1 \right]}
\newcommand{\curlyBrackets}[1]{\left\{ #1 \right\}}
\newcommand{\mathify}[1]{$#1$}

\newcommand{\ALPHA}{\mathify{\alpha}}
\newcommand{\alphaopt}{\alpha^\star}
\newcommand{\ALPHAOPT}{\mathify{\alphaopt}}
\newcommand{\localmean}[1][X]{\bar #1 }
\newcommand{\LOCALMEAN}[1][X]{\mathify{\localmean[#1]}}
\newcommand{\wavg}[1][\alpha]{\bar \mu_{#1}}
\newcommand{\WAVG}[1][\alpha]{\mathify{\wavg[#1]}}
\newcommand{\OPTAVG}{\WAVG[\alphaopt]}
\newcommand{\LOCALMEANs}{\LOCALMEAN{}  and \LOCALMEAN[Y]}
\newcommand{\samples}[1][X]{\curlyBrackets{\MakeLowercase{#1}_1, \dots, \MakeLowercase{#1}_{n_{#1}} }}
\newcommand{\SAMPLES}[1][X]{\mathify{\samples[#1]}}
\newcommand{\sampling}[1][X]{\samples[#1] \overset{i.i.d.}{\sim} #1}
\newcommand{\SAMPLING}[1][X]{\mathify{\sampling[#1]}}
\newcommand{\SAMPLINGXY}{Let \SAMPLING{}, \SAMPLING[Y]}
\newcommand{\truemean}[1][X]{\mu_{#1}}
\newcommand{\TRUEMEAN}[1][X]{\mathify{\truemean[#1]}}

\newcommand{\expectation}[1][X]{\mathbb{E}\!\left[ #1 \right]}
\newcommand{\bias}{\squarediff{\truemean[Y] }{\truemean }} 
\newcommand{\BIAS}{\mathify{\bias}}
\newcommand{\variance}[1][X]{\operatorname{Var} \!\left[ #1 \right] }
\newcommand{\VARIANCE}[1][X]{\mathify{\variance[#1]}}
\newcommand{\varshort}[1][X]{\sigma_{#1}^2}
\newcommand{\VARSHORT}[1][X]{\mathify{\varshort[#1]}}
\newcommand{\squarediff}[2]{\brackets{ #1 - #2 }^2}
\newcommand{\eseformula}[1][\wavg]{\expectation[\squarediff{#1}{\truemean}] }
\newcommand{\ese}[1][\alpha]{e_{#1}}
\newcommand{\ESE}[1][\alpha]{\mathify{\ese[#1]}}
\newcommand{\optese}{\ese[\alphaopt]}
\newcommand{\OPTESE}{\ESE[\alphaopt]}
\newcommand{\ESETEXTFULL}{expected squared error}
\newcommand{\ESETEXT}{ESE}
\newcommand{\dd}[1][\alpha]{\frac{\delta}{\delta #1}}
\newcommand{\ddd}[2][\alpha]{\frac{\delta^{#2}}{\delta {#1}^{#2}}}
\newcommand{\mathTangent}[1][2\alphaopt]{t_{#1}}
\newcommand{\tangent}[1][2\alphaopt]{\mathify{\mathTangent[#1]}}

\newcommand{\scaledVar}[1][X]{\frac{\sigma_{#1}^2}{n_#1}}
\newcommand{\tscaledVar}[1][X]{\sigma_{#1}^2/n_#1}
\newcommand{\scaledVarsquared}[1][X]{\left(\scaledVar[#1]\right)^2}
\newcommand{\tscaledVarsquared}[1][X]{\left(\tscaledVar[#1]\right)^2}
\newcommand{\conditions}{\max \curlyBrackets{\squarediff{\truemean[Y]}{\truemean}, \varshort, \varshort[Y] } > 0}
\newcommand{\CONDITIONS}{\mathify{\conditions}}
\newcommand{\DISTRIBUTIONS}{$X$ and $Y$}
\newcommand{\MEANSANDVARS}{means and variances of \DISTRIBUTIONS}
\newcommand{\THEMEANSANDVARS}{the \MEANSANDVARS}
\newcommand{\subfigcaption}[1]{{$n_Y = #1 \frac{\varshort[Y]}{\varshort} n_X$}}

\section{Introduction}
\label{sec:intro}

\emph{Collaborative learning} refers to the setting where several agents, each having access to their own personal training data, collaborate in hopes of learning a better model. In most real applications, the data distributions are different for each agent. Each agent therefore faces a key decision, in order to obtain the highest quality model: Should they ignore their peers and simply train a model on their local data alone---or should they collaborate to potentially benefit from additional training data from other agents?

This fundamental question is at the center of two areas that are currently of high interest  in both research and industry applications, namely model personalization as well as  federated and decentralized  learning.

\vspace{-2mm}
\paragraph{Federated Learning.}
In federated learning (FL) \cite{konecny2016federated_first}, training data is distributed over several agents or locations.
For instance, these could be several hospitals collaborating on a clinical trial, or billions of mobile phones involved in training a next-word prediction model for a virtual keyboard application.
The purpose of FL is to train a global model on the union of all agents' individual data. The training is coordinated by a central server, while the agents' local data never leaves its device of origin.
Owing to the data locality, FL
has become the most prominent collaborative learning approach in recent years towards privacy-preserving machine learning \cite{kairouz_advances_FL_2019,FL_overview_2020}.
\emph{Decentralized learning} refers to the analogous setting without a central server, where agents communicate peer-to-peer during training, see e.g. \cite{nedic2020distributed}.

\vspace{-2mm}
\paragraph{Personalization.}
When each agent has a different data distribution\footnote{
    See \citet[Section 3.1]{kairouz_advances_FL_2019} for common types of violations of independence and identical distribution in FL.
} (and thus a different learning task), a ``one model fits all'' approach leads to poor accuracy on individual agents. Instead, a given global model (such as from FL) needs to be personalized, e.g., by additional training on the local data of our agent. 
Prominent approaches to address this important problem of statistical heterogeneity are additional local training or fine-tuning \cite{local_fine_tuning_2019}, or weighted averaging between a global model and a locally trained model, during or after training \cite{mansour_three_approaches_personalized_FL_2020,Adaptive_Personalized_FL_2020}.

\vspace{-2mm}
\paragraph{The Question.} 
In this paper, we address weighted averaging of several models, asking: \emph{When can a weighted model average outperform an individual agent's local model? How much can we gain by using it? How much can we lose?} We aim to answer this while making only minimal assumptions on the data distributions.

In the context of FL, an answer to these questions serves the users to determine under what conditions FL should be preferred to independently training a local model. It is also of interest to the FL server to identify and potentially reward participants with high contributions to the model quality.
In decentralized learning, it can be used by agents to select most compatible peers during training \cite{grimberg2020weight}.
The question also naturally extends to model personalization, as in \citet{Adaptive_Personalized_FL_2020} who ask \emph{``when is personalization better?"} and \emph{``what degree of personalization is best?"} in the different context of model interpolation. 

\vspace{-2mm}
\paragraph{Contributions.} 
In this work, we analyse the linear combination of two models for arbitrary scalar mean estimation problems.
Given a local empirical mean~\LOCALMEAN{} and some other empirical mean \LOCALMEAN[Y],
we ask whether the
\emph{weighted model average} $\brackets{1-\alpha} \localmean + \alpha \localmean[Y]$
is a better estimator of the \emph{local true mean} $\expectation[\localmean]$, than \LOCALMEAN{} itself.
For instance,~\LOCALMEAN[Y] could be a global model obtained through federated learning without the training data of \LOCALMEAN.

We calculate the error 
of the weighted model average  with respect to the local true mean, taking the expectation over~\LOCALMEANs.
We find the optimal weight \ALPHAOPT{} to minimize this error,
showing that the error of the optimally weighted average is reduced by a fraction \ALPHAOPT{} (equal to the weight itself), compared to the error of \LOCALMEAN.

In a variant of the bias-variance trade-off, we find that there is always some positive amount $\alphaopt > 0$ of model averaging that reduces the error compared to the local model \LOCALMEAN{} (provided that \LOCALMEAN{} has a non-zero variance).
Recognizing that the optimal weight depends on quantities that are likely unknown to the experimenter,
we quantify the error of a sub-optimally weighted model average with weight \ALPHA.
We show that the error is better than the local model's if $\alpha < 2\alphaopt$, and that even a small weight $\alpha < \alphaopt$ can lead to a relatively large improvement.
On the other hand, the error is worse than the local model's if $\alpha > 2\alphaopt$.
Thus, choosing an exceedingly large weight can be harmful in situations where $2\alphaopt < 1$.
This could easily happen in practice, if \ALPHA{} is chosen based on the observed data in the presence of a large sampling bias.

We introduce our model in \cref{sec:model and assumptions}, showing how it relates to practical use cases and reflecting on the assumptions underlying our results.
We prove our main results in \cref{sec:theoretical results}
and visualize and discuss them in \cref{sec:Discussion}. Here, we also interpret weighted model averaging as a general form of shrinkage and explain how our results compare to recent related work.
We conclude with a brief summary of our results and pointers to open problems in \cref{sec:Conclusion}.

\section{Related Work}
\label{subsec:related works}

\paragraph{Model Personalization.}
A properly identified source of heterogeneity between the agents
can sometimes be addressed
by incorporating context information as additional features in the ML model: location, time, user demographic, etc.
Several other solutions have been proposed to tackle unidentified sources of heterogeneity:
\vspace{-2mm}
\begin{description}[itemsep=1pt,wide]
    \item[Local fine-tuning,]
        wherein the fully trained global model is adapted to the local context by taking a number of local stochastic gradient descent (SGD) steps \cite{local_fine_tuning_2019}.
    \item[Multi-task learning and clustering,]
        to group agents and train one model per cluster.
        While agents can often be clustered
        by their context or geolocalisation, \citet{mansour_three_approaches_personalized_FL_2020} propose an expectation-maximization-type algorithm to jointly optimize the models and the clustering.
    \item[Gradient-based personalization methods,]
        that learn a personalized model by using variants of gradient similarity measures to adapt the weight of local and global models \cite{grimberg2020weight,fallah2020personalized}. 
    \item[Joint local and global learning,]
        where a global model and a local model are trained simultaneously \cite{Adaptive_Personalized_FL_2020,mansour_three_approaches_personalized_FL_2020,fallah2020personalized}, for example using model-agnostic meta-learning \cite{kairouz_advances_FL_2019}.
    \item[Merging already trained models] is also viable using approaches such as model fusion \cite{singh2020model} or distillation \cite{lin2020ensemble}, however these again come without theoretical guarantees on the quality of the resulting merged model.
\end{description}
\vspace{-2mm}
\paragraph{Theoretical Analysis of Weighted Model Averaging.}
\citet{Donahue2020stable_coalitions} investigate FL from a game-theory point of view, to test whether self-interested players (i.e., agents) have an incentive to join an FL task.
They analyze the formation of clusters \emph{(stable coalitions)} of players in a linear regression task with specific assumptions on the data generation process. Once allowing each player to use a weighted average of their local model with their cluster's global model, the \emph{grand coalition} (i.e. a federation of all agents) is weakly preferred by all players over their local model.
Further, if players can weight each other player's local model individually rather than selecting a single weight for their cluster's model, the \emph{grand coalition} becomes core stable, where no other coalition $C$ is preferred over the \emph{grand coalition} by all players in $C$.

Thus, weighted model averaging seems more promising than clustering-based approaches in the setting under consideration.
We expand on this analysis of weighted model averaging, proving that the results about the optimal model averaging weight hold even under minimal assumptions on the data generation process.
is based on unrealistic assumptions of identical variance and similar means across players.
In contrast, we allow arbitrary finite means and variances, only requiring a positive variance for the local model.
Rather, our analysis is limited to a one-dimensional mean estimation problem with two players ($D=1, \, M=2$).
As shown in \cref{subsec:equivalence with Donahue}, our results are equivalent to theirs when our respective assumptions are applied jointly.

\paragraph{The Bias-Variance Trade-Off.}
Weighted model averaging is a form of bias-variance trade-off:
It aims to leverage the vast quantity of data on the network to reduce the model's generalization error, at the expense of increased bias if the local and global distributions match poorly.
Theoretical research on the bias-variance trade-off dates back to the surprising results of \citet{james1961}, who constructed a biased empirical estimator that provably dominates the maximum likelihood estimator for a specific higher-dimensional mean estimation problem (cf. \cref{subsec:relation to James-Stein estimator}).
Today, we would call this a \emph{shrinkage} estimator.

Despite their long history, 
shrinkage estimators are still relevant and can be found in recent applications.
For instance, \citet{2020shrinkage} recently investigated how to shrink importance weights to reduce the mean squared error of doubly-robust estimators in off-policy evaluation.
In \cref{subsec:relation to James-Stein estimator}, we show that our analysis of weighted model averaging includes shrinkage as a special case, when \LOCALMEAN[Y] is set to a constant value, such as $0$.

\section{Model and Assumptions}
\label{sec:model and assumptions}

\subsection{Setup}
\label{subsec:model}
To introduce the theoretical setting, we suppose that an agent $a_X$ has drawn $n_X$ independent samples from the random variable $X$, and that $a_X$ wishes to estimate the unknown true mean $\truemean \triangleq \expectation$.
For instance, $a_X$ could be a clinical researcher interested in accurately estimating the effect size of a new treatment from the $n_X$ patients participating in a trial at their hospital.

Agent $a_X$ can construct an empirical estimator for \TRUEMEAN{} based on its own $n_X$ samples.
For instance, $a_X$ could use the local empirical mean \LOCALMEAN{} from \cref{def:empirical means}, which is unbiased and consistent.
Alternatively, agent $a_X$ can enlist the help of another agent, $a_Y$, who
has calculated the empirical mean~\LOCALMEAN[Y] 
of $n_Y$ samples drawn from a different random variable~$Y$.
This allows $a_X$ to estimate \TRUEMEAN{} by a weighted average of \LOCALMEAN{} and \LOCALMEAN[Y], denoted as \WAVG{}, where $\alpha \in \left[0, 1\right]$ is the weight of \LOCALMEAN[Y] (\cref{def:weighted average}).

\begin{definition}[Empirical mean]
    \label{def:empirical means}
    \SAMPLINGXY.
    We recall the \emph{empirical means} \LOCALMEANs: 
    \begin{equation*}
        \localmean \triangleq \frac{1}{n_X} \sum_{i=1}^{n_X} x_i, \quad \localmean[Y] \triangleq \frac{1}{n_Y} \sum_{i=1}^{n_Y} y_i
    \end{equation*}
\end{definition}
\begin{definition}[Weighted average]
    \label{def:weighted average}
    We define the \emph{weighted model average} \WAVG{} of the two means as:
    \begin{equation*}
        \wavg \triangleq \brackets{1-\alpha} \localmean + \alpha \localmean[Y]
    \end{equation*}
\end{definition}

\paragraph{Examples of Helper Agent $a_Y$.}
In our example of a clinical trial, the helper agent $a_Y$ could be another researcher conducting a trial of a similar treatment in a different location.
Alternatively, \LOCALMEAN[Y] could also be a global model obtained through federated learning by several other hospitals conducting treatment trials.
As we do not make any assumptions on $Y$, it could also simply be an arbitrary constant (cf. \cref{subsec:relation to James-Stein estimator}).
For instance, picking $Y = 0$ results in the shrinkage estimator $\wavg = \brackets{1 - \alpha} \localmean$.

\subsection{Problem Definition}
\label{subsec:task}
We assume that $a_X$ wants to use the estimator whose \emph{expected squared error} (\ESETEXT{}) with respect to 
\TRUEMEAN{} is minimal.
\emph{Optimality} is therefore understood throughout this paper in terms of the \ESETEXT.
Note that the estimator \WAVG{} is biased for $\alpha>0$ unless $\expectation[Y] = \truemean$, see \cref{def:ese}.

\begin{definition}[Expected squared error]
    \label{def:ese}
    Recall that $\truemean \triangleq \expectation$.
    By \ESE, we denote the \ESETEXT{} of \WAVG{} w.r.t. \TRUEMEAN.
    \begin{equation*}
        \ese = \eseformula 
    \end{equation*}
\end{definition}
This focus on the estimator's \ESETEXT{} implies that $a_X$ is equally concerned about underestimating the mean, as about overestimating it---which need not be the case in practice.
Furthermore, by focusing on the squared error, we assume that $a_X$ prefers a high probability of incurring a relatively small error, than an $n$ times lower probability of incurring an $n$ times larger error.

The task of $a_X$ reduces to selecting an optimal weight \ALPHAOPT{} (\cref{def:alpha*}).
Indeed, \LOCALMEANs{} can be expressed as \WAVG[0] and  \WAVG[1], respectively.
The \emph{global model}, i.e. the empirical mean over the union of $a_X$ and $a_Y$'s samples, is obtained by selecting $\alpha = n_Y/\brackets{n_X + n_Y}$.

The global model is optimal if $X$ and $Y$ have the same mean and variance, because it is the weight for which the variance of \WAVG{} is minimal.
However, a greater weight can be optimal (e.g., if $\expectation[Y] = \expectation$ and $\variance[Y] < \variance$), whereas a smaller weight is optimal if the true means of $X$ and $Y$ are very dissimilar or if $\variance[Y] > \variance$.

To quantify the optimal weight \ALPHA{} for any distributions $X$ and $Y$, we express the \ESETEXT{} of \WAVG{} as a function of \ALPHA.
We answer the questions:
\textit{Which weight \ALPHAOPT{} would we tell $a_X$ to use, if we had perfect knowledge of $X$ and $Y$?
Exactly how much smaller would the error \OPTESE{} be, relative to \ESE[0]?
How large would the error \ESE{} be if $a_X$ used a different weight \ALPHA?}

\begin{definition}[Optimal model averaging weight]
    \label{def:alpha*}
    We denote \ALPHAOPT{} the weight which minimizes the \ESETEXT{} of \OPTAVG{} w.r.t. $\expectation$:
    \begin{equation*}
        \alphaopt \triangleq \underset{\alpha}{\arg \min} \left( \ese \right)
    \end{equation*}
\end{definition}

\subsection{Assumptions on $X$ and $Y$}
\label{subsec:assumptions on X and Y}
We make \emph{no assumptions} on the data distributions $X$ and $Y$---not even that they are related.
However, our results in their strongest form depend on $\variance$ being non-zero, and on $X$ and $Y$ having a finite mean and variance. Instead of making specific assumptions on $X$ and $Y$, we express our results as functions of the unknown quantities $\squarediff{\expectation[Y]}{\expectation}$, $\variance/n_X$, and $\variance[Y]/n_Y$.
This approach allows us to quantify the potential of weighted model averaging exactly, for any distributions $X$ and $Y$ (where $Y$ could be the union of all other agents' data sets in FL).
On the other hand, since our results depend on unknown quantities, they cannot be used directly in empirical estimation problems.
While it is certain that no empirical estimator based on the weighted averaging of \LOCALMEANs{} can do better than the optimal estimator \OPTAVG, it is much less clear that this bound can actually be attained in practice.

\section{Theoretical Results}
\label{sec:theoretical results}

In \cref{thm:alpha_opt}, we find an analytic expression for the optimal model weight \ALPHAOPT.
We conclude in \cref{corollary:alpha>0}, that the local empirical mean is dominated unless $X$ is deterministic.
Subsequently, we give the \ESETEXT{} of the weighted average \WAVG{} for any \ALPHA{} in \cref{thm:ese of any alpha},
finding its minimum in \cref{corollary:ese of alpha*} and its maximum in \cref{corollary:max ese}.

\subsection{Prelude: Estimator Properties}
\label{subsec:estimators}
In \cref{lemma:expectation and variance of empirical mean,lemma:expectation and variance of weighted average,lemma:ese of local estimator,lemma:formula for ese and convexity}, we calculate statistical properties of the estimators introduced in \cref{sec:model and assumptions}, including their \ESETEXT.
The proofs of these basic results are relegated to \cref{sec:lemma proofs for theoretical results}.

\begin{lemma}
    \label{lemma:expectation and variance of empirical mean}
    Let $\expectation = \truemean$, $\expectation[Y] = \truemean[Y]$, $\variance = \varshort$, and $\variance[Y] = \varshort[Y]$.
    We recall the properties of the empirical means \LOCALMEANs:
    \begin{align*}
        \expectation[\localmean] &= \truemean,  & \expectation[\localmean[Y]] &= \truemean[Y]  \\
        \variance[\localmean]    &= \scaledVar, & \variance[\localmean[Y]]    &=  \scaledVar[Y]
    \end{align*}
\end{lemma}

\begin{lemma}
    \label{lemma:expectation and variance of weighted average}
    Recall the nomenclature of \cref{lemma:expectation and variance of empirical mean}. The weighted average \WAVG{} has the following properties:
    \begin{alignat*}{2}
        \expectation[\wavg] &= \brackets{1-\alpha} \truemean &&+ \alpha \truemean[Y] \\
        \variance[\wavg] &=
        \brackets{1-\alpha}^2 \scaledVar &&+ \alpha^2 \scaledVar[Y]
    \end{alignat*}
\end{lemma}

\begin{lemma}
    \label{lemma:ese of local estimator}
    We compute the \ESETEXT{} of the empirical means:
    \begin{align*}
        \ese[0]
        &= \eseformula[\localmean]
        = \scaledVar
        \\
        \ese[1]
        &= \eseformula[\localmean[Y]]
        = \squarediff{\truemean[Y]}{\truemean}  + \scaledVar[Y]
    \end{align*}
\end{lemma}

\begin{lemma}
    \label{lemma:formula for ese and convexity}
    We find the \ESETEXT{} of \WAVG{} analytically below. It is convex w.r.t. \ALPHA. If \CONDITIONS, then it is strictly convex.
    \[
        \ese
        = \brackets{1-\alpha}^2 \ese[0] + \alpha^2 \ese[1]
    \]
\end{lemma}

\subsection{Optimal Model Averaging Weight}
\label{subsec:optimality}

In \cref{subsec:task}, we defined the goal of agent $a_X$ as selecting an optimal weight \ALPHAOPT{} to minimize the \ESETEXT{} of \WAVG{} w.r.t. \TRUEMEAN.
We prove the existence and (under minimal assumptions) uniqueness of \ALPHAOPT{} in \cref{thm:alpha_opt}, where we also find its analytical form.
This result is illustrated in \cref{fig:alpha}.
In \cref{corollary:alpha>0}, we conclude that, unless $\variance = 0$, the empirical mean \LOCALMEAN{} is always dominated by \emph{some} weighted average \WAVG{} with $\alpha > 0$.
Finally, in \cref{corollary:two upper bounds on alpha*} we prove an upper bound on \ALPHAOPT{} that is simpler than \cref{thm:alpha_opt}.

\begin{theorem}
    \label{thm:alpha_opt}
    Recall the setting and the estimators from \cref{def:empirical means,def:weighted average,def:alpha*}, as well as the nomenclature of \cref{lemma:expectation and variance of empirical mean}.
    There \textbf{exists} an $\alphaopt \in \left[0,1\right]$ that minimizes the \ESETEXT{} of \WAVG{} w.r.t. \TRUEMEAN.
    Further, if \CONDITIONS, then \ALPHAOPT{} is \textbf{unique} and satisfies the following equality:
    \begin{equation*}
        \label{eq:alpha_opt}
        \alphaopt
        = \frac{\ese[0]}{\ese[0] + \ese[1]}
        \left( = \frac{\variance[\localmean] }{ \variance[\localmean] + \bias  + \variance[\localmean[Y]] } \right)
    \end{equation*}
\end{theorem}
\begin{proof}[Proof of \cref{thm:alpha_opt}]
    Recall from \cref{lemma:formula for ese and convexity} that the error is convex in \ALPHA,
    from which we conclude that it admits at least one global minimum.
    \textbf{If $\brackets{\truemean[Y] - \truemean}^2 = 0 = \varshort = \varshort[Y]$,} then
    the error vanishes irrespective of \ALPHA{} (by \cref{lemma:formula for ese and convexity}).
    Thus, any \ALPHA{} minimizes the error, including all $\alpha \in \left[0,1\right]$.
    \textbf{Otherwise,} the error is strictly convex (by \cref{lemma:formula for ese and convexity}) and is uniquely minimized by \ALPHAOPT{} s.t. $0 = \dd \ese$.
    Thus, we find:
    \begin{align*}
        0
        = \dd \ese
        &= 2 \left(\alphaopt - 1\right) \ese[0] + 2 \alphaopt \ese[1]
        \\
        \Leftrightarrow
        \quad
        2 \ese[0]
        &= 2 \alphaopt \brackets{\ese[0] + \ese[1]}
        \\
        \overset{\max \curlyBrackets{\ese[0], \ese[1]} > 0}{\Leftrightarrow}
        \alphaopt
        &= \frac{\ese[0]}{\ese[0] + \ese[1]}
    \end{align*}
    In the last line above, we have made use of the assumption that \CONDITIONS, which implies that \ESE[0] or \ESE[1] is strictly positive.

    To prove that $\alphaopt \in \left[0,1\right]$, we rewrite the last equation to obtain (using the non-negativity of the \ESETEXT):
    \begin{equation*}
        \alphaopt
        = \begin{cases}
            0, & \text{if } \ese[0] = 0 \\
            0 \leq \frac{1}{1 + \ese[1] / \ese[0]} \leq 1, & \text{otherwise}
        \end{cases}
        \qedhere
    \end{equation*}
\end{proof}

Using \cref{thm:alpha_opt}, we can ask exactly when the empirical mean $\wavg[0] = \localmean$ is optimal, leading to the following corollary:

\begin{theoremcorollary}
    \label{corollary:alpha>0}
    Let \THEMEANSANDVARS{} be finite, and let $\varshort > 0$.
    Then, the empirical mean $\wavg[0] = \localmean$ is dominated by some $\wavg[\alphaopt]$ with $\alphaopt > 0$.
    Formally:
    \begin{equation*}
        \left(
            \varshort \neq 0
        \right)
        \  \text{and} \ 
        \left(
            \max \curlyBrackets{\ese[0], \ese[1]} < \infty
        \right)
        \implies \alphaopt > 0
    \end{equation*}
\end{theoremcorollary}
\begin{proof}[Proof of \cref{corollary:alpha>0}]
    By \cref{lemma:ese of local estimator}, \ESE[0] and \ESE[1] are finite if \THEMEANSANDVARS{} are finite.
    This, in turn, implies that the denominator of \ALPHAOPT{} is bounded.
    Additionally, \ESE[0] is strictly positive if $\varshort \neq 0$.
    Hence, we conclude from \cref{thm:alpha_opt} that \ALPHAOPT{} is strictly positive under the stated conditions.
\end{proof}

Finally, we can prove simpler bounds on \ALPHAOPT{} to improve our intuition:
\ALPHAOPT{} is only high when the local model has a lot of variance to trade away ($\variance[\localmean]$ is large) and when \LOCALMEAN[Y] has a low variance and is not unreasonably biased ($\variance[\localmean[Y]]$ and \BIAS{} are small):
\begin{theoremcorollary}
    \label{corollary:two upper bounds on alpha*}
    Let $\varshort > 0$. Then:
    \begin{equation}
        \label{eq:two upper bounds on alpha*}
        \alphaopt <  \min\left\{
            \frac{\variance[\localmean]}{\bias}
            ,
            \frac{\variance[\localmean]}{\variance[\localmean[Y]]}
        \right\}
    \end{equation}
\end{theoremcorollary}
\begin{proof}[Proof of \cref{corollary:two upper bounds on alpha*}]
    $\varshort > 0$ implies $\variance[\localmean] > 0$, thus we write \cref{thm:alpha_opt} as:
    \begin{equation}
    \label{eq:alpha* with only two terms}
        \alphaopt
        = \tfrac{1}{1 + \frac{\bias}{\variance[\localmean]} + \frac{\variance[\localmean[Y]]}{\variance[\localmean]}}
        < \frac{1}{\frac{\bias}{\variance[\localmean]} + \frac{\variance[\localmean[Y]]}{\variance[\localmean]}}
    \end{equation}
    The result follows by observing that $\variance[\localmean]$, \BIAS{}, and $\variance[\localmean[Y]]$ are non-negative.
    By abuse of notation, we allow division by 0 in \cref{eq:two upper bounds on alpha*} and equate it to $+\infty$.
\end{proof}

\subsection{Expected Squared Error}

Armed with a formula for \ALPHAOPT{}, we now investigate the effect of (optimally) weighted model averaging on the \ESETEXT.
\Cref{thm:ese of any alpha}, which is illustrated in \cref{fig:ese of wavg} (\cref{sec of fig ese of wavg}), relates the error of any \WAVG{} to that of \LOCALMEAN.
In \cref{corollary:ese of alpha*,corollary:ese of 2alpha*,corollary:ese of 1}, we calculate the \ESETEXT{} of \OPTAVG, \WAVG[2\alphaopt], and \LOCALMEAN[Y].
Finally, we prove in \cref{corollary:max ese} that the error of \OPTAVG{} is bounded by \ESE[0] and \ESE[1].

\begin{theorem}
    \label{thm:ese of any alpha}
    If \ALPHAOPT{} is unique and positive (i.e., if $\varshort >0$), then we can express \ESE{} as:
    \[
        \ese = \brackets{1 + \alpha \brackets{\frac{\alpha}{\alphaopt} - 2}} \ese[0]
    \]
\end{theorem}
\begin{proof}[Proof of \cref{thm:ese of any alpha}]
    We use our assumption ($\varshort > 0 \implies \ese[0] > 0$) to express \ESE[1] as a function of \ESE[0]:
    \begin{multline*}
        \ese[1]
        = \frac{\ese[0]}{\ese[0]} \frac{\brackets{\ese[0] + \ese[1]} \ese[1]}{\ese[0] + \ese[1]}
        = \ese[0] \frac{\ese[0] + \ese[1]}{\ese[0]} \brackets{1 - \frac{\ese[0]}{\ese[0] + \ese[1]}}
        \\
        = \ese[0] \frac{1-\alphaopt}{\alphaopt}
        = \ese[0] \brackets{\frac{1}{\alphaopt} - 1}
    \end{multline*}
    Inserting into \cref{lemma:formula for ese and convexity} yields:
    \begin{multline*}
        \ese
        = \brackets{1-\alpha}^2 \ese[0] + \alpha^2 \ese[0] \brackets{\frac{1}{\alphaopt} - 1}
        \\
        = \ese[0] \brackets{1 - 2\alpha + \alpha^2 + \alpha^2 \brackets{\frac{1}{\alphaopt} - 1}}
        \\
        = \ese[0] \brackets{1 + \alpha \brackets{-2 + \frac{\alpha}{\alphaopt}}}
        \qedhere
    \end{multline*}
\end{proof}
Inserting \ALPHAOPT{} into \cref{thm:ese of any alpha}, we obtain the minimum error of weighted model averaging:
\begin{theoremcorollary}
    \label{corollary:ese of alpha*}
    Estimating \TRUEMEAN{} by $\wavg[\alphaopt]$ instead of \LOCALMEAN{} leads to a $(100\cdot\alphaopt) \%$ reduction in \ESETEXT{} if $\varshort > 0$.
    Formally:
    \[
        \optese = \brackets{1 - \alphaopt } \ese[0]
        \, \  = \min_{\alpha \in \squareBrackets{0,1}} \ese
    \]
\end{theoremcorollary}
\begin{proof}[Proof of \cref{corollary:ese of alpha*}]
    The result follows from \cref{thm:ese of any alpha}, by observing that $\brackets{\frac{\alphaopt}{\alphaopt} - 2} = -1$.
    The equality $ \min_{\alpha \in \squareBrackets{0,1}} \ese = \optese$ follows from \cref{thm:alpha_opt}.
\end{proof}

We obtain \ESE[2\alphaopt] and \ESE[1] analogously by evaluating \cref{thm:ese of any alpha}.
\Cref{corollary:ese of 2alpha*} shows that $\alpha = 2\alphaopt$ is no better than using \LOCALMEAN, while \cref{corollary:ese of 1} shows that the largest possible weight $\alpha = 1$ does better than \LOCALMEAN{} if $\alphaopt > \frac{1}{2}$, and worse if $\alphaopt < \frac{1}{2}$.
\begin{theoremcorollary}
    \label{corollary:ese of 2alpha*}
    If $\varshort > 0$, then the \ESETEXT{} of \WAVG[2 \alphaopt] (w.r.t. \TRUEMEAN) is equal to that of \LOCALMEAN. Formally:
    \[
        \ese[2\alphaopt] = \ese[0]
    \]
\end{theoremcorollary}

\begin{theoremcorollary}
    \label{corollary:ese of 1}
    If $\varshort > 0$, then the \ESETEXT{} of \LOCALMEAN[Y] is given by:
    \[
        \ese[1] = \brackets{ \frac{1}{\alphaopt} - 1} \ese[0]
    \]
\end{theoremcorollary}

As a counterpart to the global lower bound on \ALPHAOPT{} found in \cref{corollary:ese of alpha*}, we prove a global upper bound below:
\begin{theoremcorollary}
    \label{corollary:max ese}
    The \ESETEXT{} of \WAVG{} is bounded above by:
    \[
        \max_{\alpha \in \squareBrackets{0,1}} \ese
        = 
        \begin{cases}
            \ese[0], & \text{if $\varshort > 0$ and  $\alphaopt \geq \frac{1}{2}$} \\
            \ese[1], & \text{otherwise}
        \end{cases}
    \]
\end{theoremcorollary}
\begin{proof}[Proof of \cref{corollary:max ese}]
    By its convexity (\cref{lemma:formula for ese and convexity}), \ESE{} admits a maximum on the closed interval $\squareBrackets{0,1}$, specifically at one of the extreme points of the interval.
    Thus, we need only compare \ESE[0] and \ESE[1].
    \textbf{If $\varshort = 0$,} then $\ese[0] = 0 \leq \ese[1]$. This implies that $ \max_{\alpha \in \squareBrackets{0,1}} \ese = \ese[1]$.
    \textbf{Otherwise,} we use \cref{corollary:ese of 1} to find:
    \begin{equation*}
        \ese[0] \geq \ese[1]
        \Longleftrightarrow
        1 \geq \frac{1}{\alphaopt} - 1
        \Longleftrightarrow
        2 \alphaopt \geq 1
        \qedhere
    \end{equation*}
\end{proof}

In summary, \cref{thm:alpha_opt,corollary:alpha>0} show that there exists an optimal model averaging weight \ALPHAOPT{} and that it is unique and positive.
The effects of  model averaging on the \ESETEXT{} are quantified in
\cref{thm:ese of any alpha} and
its corollaries.

\section{Implications of our Findings}
\label{sec:Discussion}

Throughout this section, we make the assumption that $\varshort > 0$.
Indeed, if $X$ had zero variance, the local empirical mean \LOCALMEAN{} would neither need, nor permit, any further improvement.

We summarize our main results about the \ESETEXT{} of weighted model averaging in \cref{subsec:Discussion of ese results}, and those about the optimal weight of \LOCALMEAN[Y] in \cref{subsec:Discussion of alpha*}.
Then, we present numerical examples in \cref{subsec:Examples} to illustrate the various dependencies of \ALPHAOPT.
In \cref{subsec:relation to James-Stein estimator}, we show that model averaging generalizes \emph{shrinkage} estimators.
Finally, in \cref{subsec:equivalence with Donahue}, we discuss the similarities and differences between our results and those obtained by \citet{Donahue2020stable_coalitions}.

\subsection{Quality of the Weighted Model Average}
\label{subsec:Discussion of ese results}

We compute the \ESETEXT{} of the weighted average \WAVG, for any distributions \DISTRIBUTIONS{} and for any weight \ALPHA, in \cref{lemma:formula for ese and convexity}.
In \cref{thm:ese of any alpha}, we express it as a function of the weight \ALPHA, the optimal weight \ALPHAOPT, and \ESE[0] (the error of the local empirical mean \LOCALMEAN).
This dependency is illustrated in \cref{fig:ese of wavg}.

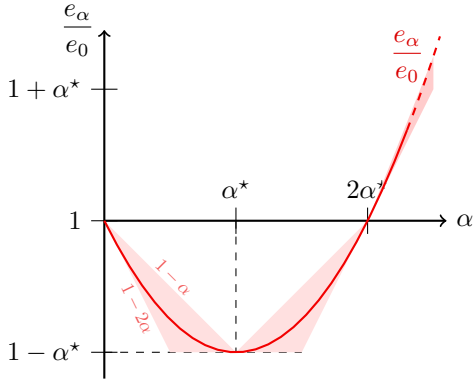
\begin{figure}[htb]
    \centering
    \begin{tikzpicture}[scale=1.75]
        \draw[->,thick] (0, 0.0) -- (2.6, 0.0) node[right] {\ALPHA};
        \draw[->,thick] (0,-1.2) -- (0.0, 1.45) node[left] {$\dfrac{\ese}{\ese[0]}$};
        
        \foreach \x/\xtext in {1/\alphaopt, 2/2\alphaopt}
            \draw[shift={(\x,0)}] (0,-0.1) -- (0,0.1) node[above] {$\xtext$};
        \foreach \y/\ytext in {-1/1-\alphaopt, 0/1, 1/1+\alphaopt}
            \draw[shift={(0,\y)}] (0.1,0) -- (-0.1,0) node[left] {$\ytext$};
        
        \draw[dashed] (0, -1) -- (1.5, -1);
        \draw[dashed] (1,0) -- (1,-1);
        
        \fill[red!12] (0, 0) -- (0.5,-1) -- (1,-1) -- cycle; 
        \fill[red!12] (1,-1) -- (1.5,-1) -- (2, 0) -- cycle; 
        \fill[red!20] (2, 0) -- (2.5, 1) -- (2.5, 1.25) -- cycle;
        
        \draw (0.3, -0.35) node[above right, rotate=-45, red!85!gray!70, scale=0.7] {$1-\alpha$}; 
        \draw (0.45,-0.85) node[below left,  rotate=-64, red!85!gray!70, scale=0.7] {$1-2\alpha$};  
            
        \draw[red!95!black,thick,dashed, domain=2.3:2.55] plot (\x, {\x*\x - 2*\x});
        \draw[red!95!black,thick, domain=0:2.3] plot (\x, {\x*\x - 2*\x});
        \draw (2.5, 1.25) node[left, red!85!black] {$\dfrac{\ese}{\ese[0]}$};
    \end{tikzpicture}
    \caption{
        Illustration of \cref{thm:ese of any alpha}.
        In red, we plot $\frac{\ese}{\ese[0]}$  depending on \ALPHA{} and \ALPHAOPT.
        The shading represents simple linear bounds. The left bounds are annotated with the corresponding formulae, and the right shape is symmetrical.
        \ESE{}: \ESETEXT{} of the weighted average \WAVG.
        \ESE[0]: \ESETEXT{} of the local estimator \LOCALMEAN.
    }
    \label{fig:ese of wavg}
\end{figure}
\label{sec of fig ese of wavg}

In \cref{corollary:ese of alpha*}, we conclude that \WAVG{} can have an \ESETEXT{} lower than that of \LOCALMEAN{} by up to a fraction \ALPHAOPT, and no lower.
However, the optimal weight \ALPHAOPT{} is a function of the (unknown) \MEANSANDVARS.
It is unclear whether it is possible to achieve the same \ESETEXT{} with a purely empirical model average, whose weight does not depend on unknown distribution parameters.

Nevertheless, even a sub-optimally weighted model average can significantly improve upon the local model \LOCALMEAN{} if \ALPHAOPT{} is not too small.
Indeed, \cref{fig:ese of wavg} illustrates that using \WAVG{} with $\alpha < \alphaopt$ reduces the error by more than a fraction of \ALPHA, and that any $0 < \alpha < 2\alphaopt$ has $\ese < \ese[0]$.
Specifically, all $0< \alpha < 1$ yield an improvement over $\alpha = 0$ if $\variance[\localmean] \geq \bias + \variance[\localmean[Y]]$ (because this implies $2\alphaopt \geq 1$).

What \cref{fig:ese of wavg} cannot show, however, is the tight upper bound to \ESE{} proven in \cref{corollary:max ese}.
For $\alphaopt \geq \frac{1}{2}$, \ESE{} is upper-bounded by \ESE[0].
Otherwise, it is upper-bounded by \ESE[1], which is related to \ESE[0] and \ALPHAOPT{} by \cref{corollary:ese of 1}.
Thus, this upper bound can be added to the figure for a given \ALPHAOPT{} by drawing a vertical line at $\alpha = 1$.
If $\alphaopt \geq \frac{1}{2}$, then the vertical line crosses $\frac{\ese}{\ese[0]}$ at a value below $1$.
\Cref{corollary:ese of 1} shows that the error can be increased substantially over that of \LOCALMEAN{}, especially if \ALPHAOPT{} is very close to $0$.

\subsection{Optimal Model Averaging Weight}
\label{subsec:Discussion of alpha*}
In \cref{subsec:Discussion of ese results}, we show that \ALPHAOPT{} is an upper bound to the potential benefits of weighted model averaging,
and that a small value of \ALPHAOPT{} opens up the possibility of doing worse than \LOCALMEAN.
This central role motivates us to visualize \ALPHAOPT{} and its dependency on the relevant parameters. In \cref{fig:alpha},
the contour lines help appreciate \cref{corollary:two upper bounds on alpha*}, which states that \ALPHAOPT{} is bounded above by the lower of the two terms:
$\left. \variance[\localmean] \right/ \bias$
(on the horizontal axis), and
$\left. \variance[\localmean] \right/ \variance[\localmean[Y]]$
(on the vertical axis).

\begin{figure}[tb]
    \centering
    \includegraphics{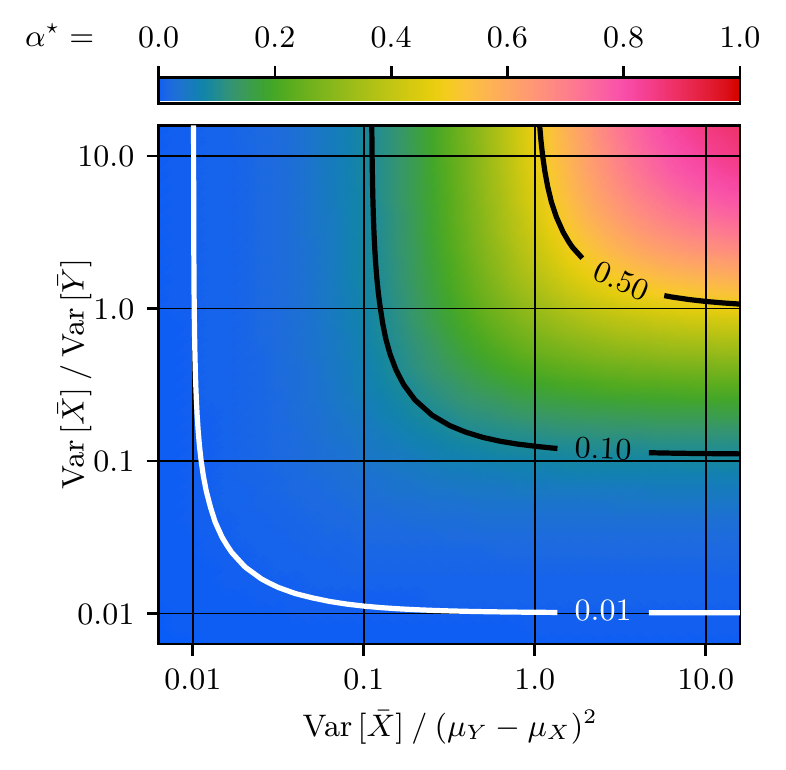}
    \caption{
        Illustration of \cref{thm:alpha_opt},
        showing the value of \ALPHAOPT{} depending on $\frac{\variance[\localmean]}{\bias}$ and $\frac{\variance[\localmean]}{\variance[\localmean[Y]]}$
        (cf. \cref{corollary:two upper bounds on alpha*}).
        The value of \ALPHAOPT{} is mapped to a color scale, and contour lines are drawn at $\alphaopt = 0.01$, $\alphaopt = 0.1$, and $\alphaopt = 0.5$.
    }
    \label{fig:alpha}
\end{figure}

Indeed, the contour lines are approximately horizontal when $\frac{\variance[\localmean[Y]]}{\bias} > 5$ (bottom right corner of \cref{fig:alpha}), and vertical when $\frac{\variance[\localmean[Y]]}{\bias} < \frac{1}{5}$ (top left corner).
What's more, the contour lines for $\alphaopt = 0.01$ and $\alphaopt = 0.1$ approximately asymptotically approach the $0.01$ and $0.1$-grid-lines respectively, as the distance from the diagonal increases.
Thus, the bound in \cref{corollary:two upper bounds on alpha*} is approximately tight as the relative difference between $\variance[\localmean[Y]]$ and $\squarediff{\truemean[Y]}{\truemean}$ increases.
This approximation is particularly good if the bound is much smaller than $1$, and falls apart of course when it is greater than $1$---as exemplified by the contour line at $\alphaopt = 0.5$.

As a consequence of this bound, increasing $n_Y$ yields diminishing returns: Doubling $n_Y$ almost doubles \ALPHAOPT{} when both $\variance[\localmean]$ and \BIAS{} are much smaller than $\variance[\localmean[Y]]$, but \ALPHAOPT{} is barely affected if one of these terms is much greater than $\variance[\localmean[Y]]$.
If the variance of \LOCALMEAN{} is much greater than that of \LOCALMEAN[Y], however, \ALPHAOPT{} may already be close to $1$.
In this case, a further increase of $n_Y$ may no longer produce a big change in \ALPHAOPT, but it may still reduce \OPTESE{} quite drastically by lowering the term $\brackets{1 - \alphaopt}$.

Importantly, both terms in \cref{corollary:two upper bounds on alpha*} feature the variance of \LOCALMEAN{} (which equals \ESE[0]).
Naturally, as the local estimator grows more accurate with increasing $n_X$, the need for collaboration vanishes along with its potential benefits.
Moreover, even for small $n_X$, model averaging is useful only if \VARSHORT{} is not too small compared to \BIAS{} and \VARSHORT[Y].

\subsection{Numerical Examples}
\label{subsec:Examples}
To give some intuition of its dependencies, we calculate \ALPHAOPT{} for a variety of scenarios in \cref{tab:examples}.
We also use the value of \ALPHAOPT{} to calculate the relative size of the error for three estimators, compared to the error of \LOCALMEAN:
\begin{description}[nosep,labelwidth=7mm]
    \item[\OPTAVG] the optimally weighted model average,
    \item[{\WAVG[1/5]}] the average with a weight of 20\% (i.e., $\tfrac{4}{5} \localmean + \tfrac{1}{5} \localmean[Y]$),
    \item[{\WAVG[1/2]}] the average with a weight of 50\% (i.e., $\tfrac{1}{2} \localmean + \tfrac{1}{2} \localmean[Y]$).
\end{description}

We describe each scenario by four easily interpretable quantities (\cref{tab:examples}, left of the vertical bar):
\begin{description}[nosep,labelwidth=7mm]
    \item[$\frac{\bias}{\varshort}$]
        the squared bias of $Y$ w.r.t. $X$, relative to \VARIANCE,
    \item[$n_X$]
        the number of samples drawn from $X$,
    \item[$\frac{\varshort[Y]}{\varshort}$] 
        the variance of $Y$, relative to that of $X$,
    \item[$\frac{n_Y}{n_X}$]
        the number of samples drawn from $Y$, relative to $n_X$.
\end{description}

However, several configurations of these four variables are equivalent because \ALPHAOPT{}
is fully determined by the terms $\frac{\bias}{\varshort} \cdot n_X$ and
$\big( \frac{\varshort[Y]}{\varshort} \big) \big/ \big( \frac{n_Y}{n_X} \big)$
(cf. \cref{eq:alpha* with only two terms}).
This is why the second, fourth, sixth, and twelfth ($n_X = 5$) rows of \cref{tab:examples} are equivalent to their respective predecessor.

The second row demonstrates that setting $n_Y \rightarrow \infty$ is equivalent to setting $\varshort[Y] \rightarrow 0$.
The fourth and sixth row show that the multiplication of $\varshort[Y]$ with a constant can be compensated by multiplying $n_Y$ with the same constant.
Furthermore, the third row of \cref{tab:examples} embodies the diminishing returns phenomenon described in \cref{subsec:Discussion of alpha*}: raising $n_Y$ until
$\big( \frac{n_Y}{n_X} \big) = 6 \big( \frac{\varshort[Y]}{\varshort} \big)$
already takes \ALPHAOPT{} most of the way towards its limit for $n_Y \rightarrow \infty$.
Finally, the eleventh and twelfth rows (from $n_X = 20$ to $n_X=5$) exemplify the connection between the first two columns, which is that \ALPHAOPT{} only depends on their product.
This also explains the equivalence of the last two rows of \cref{tab:examples}. It also reflects \cref{corollary:two upper bounds on alpha*}, as the influence of $n_Y$ is largest when $\frac{\bias}{\varshort} \cdot n_X$ is small, and becomes vanishingly small when the product exceeds $\sim10$ (see bottom rows).
Another interpretation is that, when $\truemean[Y] = \truemean \pm \sigma_X$ (i.e., $\bias = \varshort$), then using \TRUEMEAN[Y] as a proxy for \TRUEMEAN{} is exactly as good as using a single realization of $X$ (in terms of \ESETEXT).\footnote{
    Indeed, after drawing one sample $x_1 \sim X$, we find: $\eseformula[{\truemean[Y]}] = \squarediff{\truemean[Y]}{\truemean} = \varshort = \eseformula[x_1]$
}
\begin{table}[htb]
    \caption{
        Numerical examples for the value of \ALPHAOPT{} in a variety of scenarios.
        The corresponding errors of \OPTAVG, 
        \WAVG[1/5],
        and
        \WAVG[1/2]
        are  given as fractions (or multiples) of \ESE[0].
        Cells marked with $*$ could contain any finite positive value without affecting \ALPHAOPT,
        while empty cells have the same value as the cell above.
        $\operatorname{bias}^2$ denotes \BIAS{}.
    }
    \label{tab:examples}
    \vskip 0.15in
    \begin{center}
    \begin{small}
    \begin{sc}
    \begin{tabular}{cccc|lccc}
        \toprule
            $\frac{\operatorname{bias}^2}{\varshort}$
            & $n_X$
            & $\frac{\varshort[Y]}{\varshort}$
            & $\frac{n_Y}{n_X}$
            & ~\ALPHAOPT
            & $\frac{\optese}{\ese[0]}$
            & $\frac{\ese[1/5]}{\ese[0]}$
            & $\frac{\ese[1/2]}{\ese[0]}$
            \\
        \midrule
            0 & $*$ & 0   & $*$       & \textbf{1.0}  & 0.00 & 0.64 & 0.25 \\
                &   & $*$ & $+\infty$ & \textbf{1.0}  &      &      &      \\
            
                &   & 1   &  6       & \textbf{0.86} & 0.14 & 0.65 & 0.29 \\ 
                &   & 10  &  60      & \textbf{0.86} &      &      &      \\
            
                &   & 1   &  1        & \textbf{0.50} & 0.50 & 0.68 & 0.50 \\
                &   & 10  &  10       & \textbf{0.50} &      &      &      \\
             
            0.25 & 10 & 0 & $+\infty$ & \textbf{0.57} & 0.43 & 0.67 & 0.44 \\
                 &    & 1 & 1         & \textbf{0.44} & 0.56 & 0.69 & 0.56 \\
                 & 100& 0 & $+\infty$ & \textbf{0.04} & 0.96 & 1.64 & 6.50 \\
                 &    & 1 & 1         & \textbf{0.04} & 0.96 & 1.68 & 6.75 \\
            
               & 20 &  0  & $+\infty$ & \textbf{0.17} & 0.83 & 0.84 & 1.50 \\
            1  & 5  &  0  & $+\infty$ & \textbf{0.17} &      &      &      \\
               &    &  1  &  1        & \textbf{0.14} & 0.86 & 0.88 & 1.75 \\
               & 50 & 0   & $+\infty$ & \textbf{0.02} & 0.98 & 2.64 & 12.8 \\ 
               &    &  1  &  1        & \textbf{0.02} & 0.98 & 2.65 & 13.0  \\
            
            $*$ & $+\infty$ & $*$ & $*$ & \textbf{0.0} & 1.00 & $+\infty$ & $+\infty$ \\
            $+\infty$ & $*$ & $*$ & $*$ & \textbf{0.0} &      &           &           \\
        \bottomrule
    \end{tabular}
    \end{sc}
    \end{small}
    \end{center}
    \vskip -0.1in
\end{table}

\subsection{Relation to the James-Stein Estimator}
\label{subsec:relation to James-Stein estimator}

Weighted model averaging generalizes the fundamental statistical notion of \emph{shrinkage}. 
Conceptually, \emph{shrinkage} (towards $0$) simply corresponds to multiplying an empirical estimator with a weight $\beta < 1$.
This reduces the estimator's variance at the cost of additional bias, thus reducing its \ESETEXT{} if $\beta$ is sufficiently close to $1$ \cite{shrinkage_textbook}.
For $\varshort[Y] = 0$, we recover the case of estimator shrinkage towards an arbitrary anchor value (\TRUEMEAN[Y]).
If we pick $Y=0$, then \WAVG{} is simply a shrunken version of \LOCALMEAN: $\wavg = \brackets{1- \alpha } \localmean$.

We find that the optimal amount of shrinkage is never \emph{quite} $0$ for any combination of finite means, variances, and number of samples.
However, we see in \cref{tab:examples} that it quickly tends to $0$ as we shrink towards an increasingly unrelated value (with $\frac{\bias}{\varshort} \rightarrow \infty$), and as the quality of the local estimator \LOCALMEAN{} increases (with $n_X \rightarrow \infty$).

The first example of shrinkage is the infamous James-Stein (JS) estimator, which dominates the empirical mean for the problem of estimating from just one observation the mean of a spherically symmetrical multivariate normal random variable (RV) with three or more dimensions
\cite{shrinkage_textbook,james1961}.
While the JS estimator can be used without prior knowledge of the mean and variance of the RV, it only dominates the empirical mean under the assumptions that all coordinates of the RV (a) have the same variance, and that they are (b) normal and (c) mutually uncorrelated. 
It is also restricted to estimation in at least three dimensions.
By contrast, our results hold for arbitrary one-dimensional estimation problems, with minimal assumptions on $X$ and $Y$.
However, we express them as functions of statistics of $X$ and $Y$ that would not be available in a practical setting.

\subsection{Connection to Recent Related Work}
\label{subsec:equivalence with Donahue}

\citet{Donahue2020stable_coalitions} prove corresponding results to our \cref{thm:alpha_opt,corollary:alpha>0,corollary:ese of alpha*} in their Lemmas 6.1 and 6.3, under comparatively more complex assumptions.
Our model is simpler in three ways, making it more general for some aspects and more specific for others.

The most important difference is that our assumptions on the input data are minimal and significantly more realistic, as we assume nothing more than finite means and variances, and a non-zero variance for $X$.
In contrast, \citet{Donahue2020stable_coalitions} assume that  \TRUEMEAN[Y] and \TRUEMEAN{} are drawn independently from the same random variable with variance $\sigma^2$, and that $X$ and $Y$ have the same variance ($\varshort = \varshort[Y] = \mu_e$).

Secondly, we consider only two nodes ($M=2$), whereas their analysis includes an arbitrary number of players to study clustering approaches to model personalization. 
Nevertheless, our two-node model reproduces their \emph{coarse-grained federation} model (Section 6 in their paper) by setting:
\(\localmean[Y] = \frac{1}{N} \sum_{i=1}^M \hat{\theta}_i \cdot n_i\).
Finally, our analysis is confined to one-dimensional linear regression ($D=1$) compared to theirs in $D$ dimensions.

We will demonstrate that these corresponding results become equivalent in the special case when both respective classes of assumptions are applied jointly.

\begin{lemma}
    \label{lemma:equivalence with Donahue}
    By applying the assumptions of \citet{Donahue2020stable_coalitions} on the data generation process, to \cref{corollary:ese of alpha*}, we obtain the same result as by applying our assumption of $M=2$ nodes to the one-dimensional ($D=1$) version of their Lemma 6.3.
\end{lemma}

\Cref{lemma:equivalence with Donahue} follows from \cref{lemma:my theorem plus their assumptions,lemma:their theorem plus my assumptions}.
The proof of \cref{lemma:my theorem plus their assumptions} is relegated to \cref{sec:lemma proofs for discussion}.

\begin{lemmacorollary}
    \label{lemma:my theorem plus their assumptions}
    Under the assumptions that (i)  \TRUEMEAN[Y] and \TRUEMEAN{} are drawn independently from the same random variable~$W$ with variance $\sigma^2$ and (ii) $\varshort = \varshort[Y] = \mu_e$, \cref{corollary:ese of alpha*,thm:alpha_opt} yield:\vspace{-1mm}
    \[  \optese
      = \frac{2n_Y^2\sigma^2 + \mu_e n_Y}{n_Y \brackets{n_Y + n_X}
      + 2n_X  n_Y^2\frac{\sigma^2}{\mu_e}}
    \]
\end{lemmacorollary}

\begin{lemmacorollary}
    \label{lemma:their theorem plus my assumptions}
    Under the assumption that $M=2$, denoting $n_j$ by $n_X$ and $n_{i \neq j}$ by $n_Y$, Lemma 6.3 from \citet{Donahue2020stable_coalitions} simplifies to:
    \[  MSE
      = \frac{2n_Y^2\sigma^2 + \mu_e n_Y}{n_Y \brackets{n_Y + n_X} + 2n_X n_Y^2\frac{\sigma^2}{\mu_e}}
    \]
\end{lemmacorollary}

\begin{proof}
    \Cref{lemma:their theorem plus my assumptions} is obtained with the substitutions:
    \begin{align*}
        n_j &= n_X     
        &
        N   &= n_X + n_Y 
        \\
        \sum_{i\neq j} n_i^2 &= n_Y^2
        &
        N - n_j &= n_Y
        \qedhere
    \end{align*}
\end{proof}

\section{Conclusion}
\label{sec:Conclusion}

In this work, we have quantified the improvement of quality (reduction in \ESETEXTFULL{}) that can be achieved through weighted model averaging with peers, as compared to using exclusively local data.
Our results concern the mean estimation problem for any scalar random variable $X$, and include the case of shrinkage towards a fixed value.
While we limit our analysis to averaging between two models \LOCALMEANs, our results apply to federated learning by letting $a_Y$ denote the union of all other agents.
We derive the optimal averaging weight for the helper model~\LOCALMEAN[Y], proving that it is strictly positive unless $\variance =0$.
Through examples, we explain that this holds true even for an arbitrarily unsuitable helper model~\LOCALMEAN[Y] (thus for arbitrary non-IID client data), though the optimal weight converges towards $0$ as the difference between the expectations of $X$ and $Y$ increases.

Further, we show that model averaging with the optimal weight \ALPHAOPT{} reduces the \ESETEXT{} by a fraction equal to \ALPHAOPT, compared to the local model \LOCALMEAN.
We then analyse the \ESETEXT{} of the weighted model average, as a function of only the optimal weight \ALPHAOPT{} and the weight \ALPHA{} that is actually used.
We find that any weight $\alpha < 2\alphaopt$ reduces the \ESETEXT{} compared to \LOCALMEAN, by more than a fraction of \ALPHA{} if $\alpha < \alphaopt$, and that any weight of $\alpha > 2\alphaopt$ results in a larger \ESETEXT{}, compared to \LOCALMEAN.

Motivated by its central role, we visualize and investigate the dependencies of the optimal model averaging weight~\ALPHAOPT.
It depends mainly on the lower of two ratios, which both involve the variance of the local estimator, \VARIANCE[\localmean].
Thus, \ALPHAOPT{} 
can only be large if \VARIANCE[\localmean] is significantly larger than the variance of the helper model \LOCALMEAN[Y] and the squared difference between the expectations of $X$ and $Y$.
\paragraph{Outlook.}
We prove theorems under realistic assumptions about the distributions of $X$ and $Y$, in the limited setting of one-dimensional parameter estimation with two agents.
Future work could extend our results to multivariate parameter estimation and to an arbitrary number of agents.
Further, in addition to interpreting it \emph{as} shrinkage, we could apply shrinkage \emph{to} the weighted model average, thus obtaining $\mu_{\beta, \gamma} = \brackets{1-\beta}\localmean + \gamma \localmean[Y]$, where $ \beta \geq \gamma $.

Finally, whether our theoretical lower bound on the \ESETEXT{} of the weighted model average can possibly be achieved without perfect knowledge of the bias and variances of $X$ and $Y$, merits investigation.
It would thus be interesting to quantify the error reduction that can be achieved in practice, potentially under specific assumptions on $X$ and $Y$.

\newpage 

%% file: 0_appendix.tex
\newpage

\twocolumn[
\icmltitle{Optimal Weighted~Model~Averaging for Personalized Collaborative~Learning
            Appendix: Lemma~Proofs}
]

\appendix
\section{Lemmas in \cref{sec:theoretical results}}
\label{sec:lemma proofs for theoretical results}

\begin{proof}[Proof of \cref{lemma:expectation and variance of empirical mean}]
    The proofs for \LOCALMEAN{} apply analogously to \LOCALMEAN[Y]. 
    The unbiasedness of the empirical mean ($\expectation[\localmean] = \expectation$) follows from the linearity property of the expectation.
    The variance of \LOCALMEAN{} follows from the independence of the realizations \SAMPLES:
    \begin{multline*}
        \variance[\localmean] = \variance[\frac{1}{n_X} \sum_{i=1}^{n_X} x_i] = \frac{1}{n_X^2} \variance[\sum_{i=1}^{n_X} x_i] \\ 
        = \frac{1}{n_X^2} \sum_{i=1}^{n_X} \variance[ x_i] + \sum_{j \neq i}^{n_X} \underbrace{Cov\left(x_i, x_j\right)}_{=0 \ \left( x_i \perp  x_j\right)} = \frac{n_X \variance}{n_X^2} 
        \qedhere
    \end{multline*}
\end{proof}

\begin{proof}[Proof of \cref{lemma:expectation and variance of weighted average}]
    The expectation of \WAVG{} follows from the linearity property of the expectation.
    Its variance also results from the properties of the variance:
    \begin{multline*}
        \variance[\wavg] = \variance[\brackets{1-\alpha} \localmean + \alpha \localmean[Y]] \\
        = \brackets{1-\alpha}^2 \variance[\localmean] + \alpha^2 \variance[\localmean[Y]]
    \end{multline*}
    The result follows from \cref{lemma:expectation and variance of empirical mean}.
\end{proof}

\begin{proof}[Proof of \cref{lemma:ese of local estimator}]
    To compute the \ESETEXT{} of the empirical means,
    we make use of the bias-variance decomposition:
    \begin{align}
        \label{eq:bias-variance decomposition}
        && \expectation[{\squarediff{\hat \theta}{\theta}}]
        = \mathbb{E} \big[\hat{\theta}^2 &- 2 \theta \hat \theta + \theta^2 \big]
        \nonumber \\
        &=& \underbrace{
            \mathbb{E}\big[\hat{\theta}^2\big]
            \lefteqn{\overbrace{\phantom{
                - \left( \mathbb{E}\big[\hat{\theta}\big]\right)^2 + \left( \mathbb{E}\big[\hat{\theta}\big]\right)^2
            }}^{\pm 0}}
            - \left( \mathbb{E}\big[\hat{\theta}\big]\right)^2
        }_{}
        &+ \underbrace{
            \left( \mathbb{E}\big[\hat{\theta}\big] \right)^2
            - 2 \theta \mathbb{E}\big[\hat{\theta}\big]
            + \theta^2
        }_{} 
        \nonumber \\
        &=& Var\big[\hat \theta \big] &+ \squarediff{ \mathbb{E}\big[\hat{\theta}\big] }{\theta}
    \end{align}
    Plugging in $\truemean$, \LOCALMEANs, we obtain:
    \begin{align*}
        \eseformula[\localmean]
        &= \variance[\localmean] + \overbrace{\squarediff{\expectation[\localmean]}{\truemean}}^{0}
        \\
        \eseformula[\localmean[Y]]
        &= \variance[\localmean[Y]] + \squarediff{\expectation[\localmean[Y]]}{\truemean} & \hspace{6mm}
    \end{align*}
    The result follows from \cref{lemma:expectation and variance of empirical mean}.
\end{proof}

\begin{proof}[Proof of \cref{lemma:formula for ese and convexity}]
    We derive our analytic formula for \ALPHAOPT{} by plugging the results from \cref{lemma:expectation and variance of weighted average} into the bias-variance decomposition (\cref{eq:bias-variance decomposition}):
    \begin{multline*}
        \ese =
        \eseformula = \variance[\wavg] + \squarediff{\expectation[\wavg]}{\truemean} \\
        = \brackets{1-\alpha}^2 \scaledVar + \alpha^2 \scaledVar[Y] + \left(\alpha (\truemean[Y] - \truemean) \right)^2
    \end{multline*}
    Using \cref{lemma:ese of local estimator}, we obtain:
    \begin{equation*}
        \ese =  \brackets{1-\alpha}^2 \ese[0] + \alpha^2 \ese[1]
    \end{equation*}
    We prove the convexity results by taking the second derivative of \ESE{} w.r.t. $\alpha$. We obtain the first derivative trivially as:
    \begin{equation*}
        \dd \ese = 
        2 \left(\alpha - 1\right) \ese[0] + 2 \alpha \ese[1]
    \end{equation*}
    The second derivative of the error w.r.t. $\alpha$ follows as:
    \begin{equation*}
        \ddd{2} \ese =
        2 \brackets{ \ese[0] + \ese[1]}
        \geq 0
    \end{equation*}
    We conclude the proof by observing that the second derivative of the error w.r.t. $\alpha$ is a sum of non-negative terms.
    Thus, it is itself non-negative. By \cref{lemma:ese of local estimator}, it is positive if \CONDITIONS.
\end{proof}

\section{Lemmas in \cref{sec:Discussion}}
\label{sec:lemma proofs for discussion}

\begin{proof}[Proof of \cref{lemma:my theorem plus their assumptions}]
    We plug \cref{thm:alpha_opt} into \cref{corollary:ese of alpha*} to obtain:
    \begin{equation*}
        \optese
        = \left(1 - \alphaopt \right) \ese[0]
        = \frac{\ese[1]}{\ese[0] + \ese[1]} \ese[0]
    \end{equation*}
    Using \cref{lemma:ese of local estimator}, we obtain:
    \begin{equation*}
        \optese
        = \tfrac{\expectation[\squarediff{\truemean[Y] }{\truemean }]  + \scaledVar[Y] }{ \scaledVar + \expectation[\squarediff{\truemean[Y] }{\truemean }]  + \scaledVar[Y] } \scaledVar
    \end{equation*}
    Using assumption (ii):
    \begin{equation}
        \label{eq:after assumption ii}
        \optese
        = \tfrac{\expectation[\squarediff{\truemean[Y] }{\truemean }]  + \frac{\mu_e}{n_Y} }{ \frac{\mu_e}{n_X} + \expectation[\squarediff{\truemean[Y] }{\truemean }]  +  \frac{\mu_e}{n_Y} }  \frac{\mu_e}{n_X}
    \end{equation}
    Under assumption (i), we observe:
    \begin{multline*}
        \expectation[ \squarediff{\truemean[Y] }{\truemean } ] 
        = \variance[ { \truemean[Y] - \truemean } ] + \underbrace{ {\expectation[ { \truemean[Y] - \truemean } ]}^2 }_{ =0 } \\
        = \variance[ {\truemean[Y]} ] + \variance[ {\truemean } ]
        = 2 \variance[W] 
        = 2 \sigma^2 
    \end{multline*}
    Substituting in \cref{eq:after assumption ii} yields:
    \begin{multline*}
        \optese
        = \frac{2\sigma^2 + \frac{\mu_e}{n_Y}}{\frac{\mu_e}{n_X} + 2\sigma^2 + \frac{\mu_e}{n_Y}} \frac{\mu_e}{n_X} \\
        = \frac{2\sigma^2 + \frac{\mu_e}{n_Y}}{1 + 2n_X\frac{\sigma^2}{\mu_e} + \frac{n_X}{n_Y}} \\
        = \frac{2n_Y^2\sigma^2 + \mu_e n_Y}{n_Y^2 + n_Y n_X + 2n_X n_Y^2\frac{\sigma^2}{\mu_e}}
    \end{multline*}
\end{proof}